\documentclass{article}
\usepackage{amsmath,amssymb,amsfonts,amsthm}
\usepackage{algorithmic}
\usepackage{graphicx,color}

\usepackage[style=ieee,backend=bibtex,url=false,doi=false,isbn=false,date=year,citestyle=numeric-comp,maxbibnames=10,minbibnames=10,maxcitenames=10,mincitenames=10]{biblatex}

\bibliography{./references}

\usepackage{subcaption}
\usepackage{multirow}

\usepackage[english]{babel}

\newtheorem{theorem}{Theorem}
\newtheorem{corollary}[theorem]{Corollary}
\newtheorem{lemma}[theorem]{Lemma}
\newtheorem{proposition}[theorem]{Proposition}
\newtheorem{definition}{Definition}

\newtheorem{assumption}{Assumption}

\DeclareMathOperator*{\argmin}{argmin}

\newcommand{\mbb}[1]{\mathbb #1}
\newcommand{\mbf}[1]{\mathbf #1}
\newcommand{\mcl}[1]{\mathcal #1}

\newcommand{\R}{\mbb{R}}

\newcommand{\eDef}{\hfill \Box}

\begin{document}

\title{On the Turnpike to Design of Deep Neural Nets:\\ Explicit Depth Bounds
}

\author{Timm Faulwasser, Arne-Jens Hempel and Stefan Streif  
	\thanks{ Timm Faulwasser is with the Institute for Energy Systems, Energy Efficiency and Energy Economics, TU Dortmund University, Germany. {\tt\small timm.faulwasser@ieee.org};  Arne-Jens Hempel and Stefan Streif are with the Automatic Control and System Dynamics Laboratory, Technische Universität Chemnitz, 09126 Chemnitz, Germany. {\tt\small arne-jens.hempel@etit.tu-chemnitz.de; stefan.streif@etit.tu-chemnitz.de}}%
}
\date{}
\maketitle
\begin{abstract}
It is well-known that the training of Deep Neural Networks (DNN) can be formalized in the language of optimal control. In this context, this paper leverages classical turnpike properties of optimal control problems to attempt a quantifiable answer to the question of how many layers should be considered in a DNN. The underlying assumption is that the number of neurons per layer---i.e., the width of the DNN---is kept constant. Pursuing a different route than the classical analysis of approximation properties of sigmoidal functions, we prove explicit bounds on the required depths of DNNs based on asymptotic reachability assumptions and a dissipativity-inducing choice of the regularization terms in the training problem. Numerical results obtained for the two spiral task data set for classification indicate that the proposed estimates can provide non-conservative depth bounds. 
\end{abstract}

\section{Introduction}
{M}{achine} Learning (ML) and Optimal Control (OC)---despite being seemingly distant topics---share fruitful and established interconnections. For example, the nowadays widely applied concept of back propagation for efficient gradient computation in training of various ML algorithms naturally arises in OC disguised as adjoint/co-state dynamics, see \cite{Bryson62} and the overview~\cite{Baydin17a}.
Indeed the exponentially growing interest in deep learning and artificial neural networks was preceded by seminal results obtained in systems and control. Consider, for example, the analysis of approximation properties of sigmoidal (activation) functions \cite{Cybenko89a} or controllability and reachability properties of recurrent neural networks \cite{Sontag97}. Moreover, it  has been proposed to analyze the training of Deep Neural Nets (DNN) in an OC framework~\cite{Li17a,Esteve20}.

The classical works of \cite{Cybenko89a,Hornik91}, and a large number of follow-up papers, have shown that the universal approximation properties of multilayer feed-forward networks hold when condensed to huge single hidden layer networks. However, so-called \textit{shallow} networks  encounter problems when it comes to learning and might have also more parameters/weights then their multi-layered counterparts, cf. the empirical study \cite{Ba14}. One may conclude that there is a substantial depth vs. width trade-off in designing neural networks.

In context of training DNNs in ML applications network depth and width are usually chosen empirically by try-and-error. Alternatively, one can address the depth-design problem  within the optimization/learning process, for example, by choosing zero weights or by choosing other network architectures like deep residual networks featuring bypasses of layers, see \cite{Goodfellow-et-al-2016}. In principle, one could also re-phrase the depth-choice as a so-called hyper-parameter tuning problem and apply  Bayesian regression \cite{maclaurin-2015}.

This paper investigates the training of DNNs as an Optimal Control Problem (OCP) similar to \cite{Li17a,Esteve20}. 
This way, we do not take the functional analytic and function approximation route paved by  \cite{Cybenko89a,Hornik91} and others. Rather we regard a feed-forward neural network as a dynamic system on a finite horizon and invoke reachability properties to derive depth bounds, while the width---i.e., the number of neurons per layer---is kept constant. 
Specifically, we rely on dissipativity-based analysis techniques of OCPs which leverage turnpike properties. The term turnpike refers to a similarity property of  solutions of OCPs for varying initial conditions and horizon lengths. It has been coined by Dorfman, Solow and Samuelson~\cite{Dorfman58} and has received considerable interest in economics, see \cite{Mckenzie76,Carlson91}. There has been interest in turnpike properties and their relation to dissipation inequalities \cite{epfl:faulwasser15h,Gruene16a}, in turnpikes of PDE-constrained OCPs~\cite{Gugat16}, mixed-integer problems~\cite{tudo:faulwasser20f} and economic MPC~\cite{kit:faulwasser18c}. We refer to~\cite{tudo:faulwasser20l} for an overview. 
Moreover, the results of \cite{Esteve20} suggest analysing turnpike properties in continuous-time OCPs arising from continuous training problems of DNNs. Therein, ODE and PDE turnpike properties are established using a finite-time reachability assumption for the underlying DNNs and leveraging the structure of the optimality conditions. 

The present paper also takes a turnpike approach to the design of DNNs. In constrast to \cite{Esteve20}, we consider the training of discrete DNNs with a constant number of units per layer. We show, based on an established dissipativity assumptions for OCPs, that the training of DNNs exhibits the turnpike phenomenon. Moreover, we provide a constructive procedure to ensure that the regularization stage cost induces strict dissipativity of the OCP. Put differently, our analysis makes suggestions how to design regularization terms for the DNN training. Based on an asymptotic reachability assumption, we derive explicit bounds on the DNN depth---i.e., the number of layers---, while we limit the assumptions on the considered loss functions. To the best of our knowledge, the present paper appears to be the first, which derives explicit depth bounds for DNNs. Additionally, we  comment on the application of our results to classification problems. 
We draw upon the example of the classic Two Spiral Task (TST)  to illustrate our findings numerically. Numerical a-posteriori verification indicates that the proposed depth bounds are not overly conservative.

The remainder of the paper is structured as follows: Section \ref{sec:prob} introduced the problem statement and present preliminary results. Section~\ref{sec:results} presents the main results, while in Section~\ref{sec:example} we turn towards a numerical example. The paper concludes with a discussion and outlook in Section~\ref{sec:conc}.

\section{Training Neural Nets via Optimal Control}\label{sec:prob}
A residual neural net can be modeled as the system
\begin{equation}\label{eq:sys}
x_{k+1} = x_k + \sigma(A_k x_k + b_k), \quad x_0 = x^i\in \mbb{R}^d
\end{equation}
where the discrete (time) index $k\in \mbb{N}$ models the depth of the network and the weights $A_k \in \mbb{R}^{d\times d}$, $b_k \in \mbb{R}^d$ can be chosen.
The initial condition of the state variable $x \in \mbb{R}^d$ corresponds to data propagated through the net. Specifically, one considers the given set
\[\mbb{D} = \{ (x^1, y^1), \dots, (x^D, y^D) \}\]
of $D$ data points $(x^i, y^i)$ where the vectors $y^i \in \mbb{R}^m$ serve as labels of the training data. We denote the $x$-projection of $\mbb{D}$ as $\mbb{X}$ and, respectively, the $y$-projection as $\mbb{Y}$. The sequence of integers $k_0, k_0 +1, \dots, k_1$ is written as $\mbb{N}_{[k_0, k_1]}$. Notice that with slight abuse of notation, we write $\sigma(A_k x_k + b_k)$ to denote the element-wise application of the scalar activation function $\sigma:\mbb{R} \to \mbb{R}$. 

\subsection{Problem Statement}
Training a neural net can be understood as an optimization problem wherein, in principle, the depth of the network $N \in \mbb{N}$ and the width $\dim(x) =d$,  the weight sequences $A_k, b_k, k \in \mbb{N}_{[0,N-1]}$, and the activation function $\sigma$ are design parameters. 
Here, we a constant dimension $\dim(x) = d$, which implies that the number of neurons per layer is kept constant. Primarily, we are interested in deriving a non-conservative bound on the depth $N$.

\begin{assumption}[Properties of the activation function $\sigma$]
\label{ass:sigma}
The activation function $\sigma:\mbb{R} \to \mbb{R}$ is continuous on $\mbb{R}$ and satisfies $\sigma(0) = 0$.  	$\eDef$
\end{assumption}

Similar to \cite{Li17a,Esteve20} we will rewrite the training problem as an OCP. To this end, consider the stacked data vectors
\begin{align*}
\mbf{x}_0 &\doteq \begin{pmatrix}
x^{1\top}, \dots, x^{D\top}
\end{pmatrix}^\top,\quad
\mbf{y} \doteq \begin{pmatrix}
y^{1\top}, \dots, y^{D\top}
\end{pmatrix}^\top,
\end{align*}
which allow writing $\mbb{D}= \{\mbf{x}_0, \mbf{y}\}$.
Moreover, the stacked or ensemble variant of the dynamics \eqref{eq:sys} reads
\begin{align}
\mbf{x}_{k+1} &= \mbf{x}_k + \sigma\left((I^D \otimes A_k) \mbf{x}_k + 1^D\otimes {b}_k\right), ~\mbf{x}_0 \in \mbb{R}^{d\cdot D} \nonumber \\
& \doteq f(\mbf{x}_k, \mbf{u}_k), \hspace*{3.4cm}\quad \mbf{x}_0 \in \mbb{R}^{d\cdot D} \label{eq:stack_sys}
\end{align}
with  $\mbf{u}_k \doteq \begin{pmatrix}
\mathrm{vect}({{A}_k})^\top, {b}_k^\top
\end{pmatrix}^\top \in \mbb{R}^{d^2 +d}$ and $I^D$ is the identity matrix of $\mbb{R}^D$ and $1^D$ is vector of all ones in $\mbb{R}^D$. The initial condition can be understood as the vectorization of the available data, i.e., $\mbf{x}_0 =  \mathrm{vect}(\mbb{X})$.
For a given depth $N$  the problem can be written as the following discrete-time OCP 
\begin{subequations}\label{eq:OCP}
\begin{align}
V_N^\gamma(\mbf{x}_0) \doteq \min_{\{\mbf{u}_k\}}& \sum_{k=0}^{N-1} \ell(\mbf{x}_k, \mbf{u}_k) + \gamma\ell_{\mathrm{f}}(\mbf{x}_N, \mbf{y})\\
\text{subject to}&~ \forall k \in \mbb{N}_{[0, N-1]}\nonumber \\
\mbf{x}_{k+1} &= f(\mbf{x}_k, \mbf{u}_k), \quad \mbf{x}_0 = \mathrm{vect} (\mbb{X})\in \mbb{R}^{d\cdot D}.
\end{align}
\end{subequations}
The Lagrange term (or stage cost) $\ell:\mbb{R}^{d\cdot D} \times \mbb{R}^{d^2+d} \to \mbb{R}^+_0$ captures regularization terms, while the Mayer term (or loss function) $\ell_{\mathrm{f}}: \mbb{R}^{d\cdot D} \to \mbb{R}^+_0$ describes the quality of the neural net in terms of the given (classification/learning) task  at hand. 
Typically, one aims at minimizing the empirical loss
\begin{equation} \label{eq:empLoss}
\ell_{\mathrm{f}}(\mbf{x}_N, \mbf{y}) \doteq  \dfrac{1}{D} \sum_{i=1}^D \ell^i_{\mathrm{f}}(x_N(x^i),y^i),
\end{equation}
where $\ell^i_{\mathrm{f}}(x_N(x^i),y^i)$ denotes the loss associated to the data sample $(x^i, y^i)$.
The scalar $\gamma \in \mbb{R}^+$ is used to trade-off the importance of the regularization against the loss function. The horizon $N\in \mbb{N}$ corresponds to the depth of the net.
Optimal solutions to \eqref{eq:OCP} are denoted as $\mbf{x}^\star_k(\mbf{x}_0)$ and $\mbf{u}^\star_k(\mbf{x}_0)$ where, whenever necessary, the argument $\mbf{x}_0$ highlights the dependence on the given data $\mbb{D}$.

The key challenge in solving \eqref{eq:OCP} is that the dimensionality of the data $\mbf{x}_0$ might be large, while the weights $A_k$ and $b_k$ have to provide suitable propagation for all data points $x_0^i$, and a priori it is not clear what a suitable depth $N$ shall be. Indeed the weights should provide suitable propagation of data points even beyond $\mbb{X}$ such as to allow using the trained neural net as a predictor. This desired predictive or extrapolating capability is also known as \textit{generalization} in machine learning.   Moreover, it is well-known---in systems and control as well as in machine learning---that large values of $N$ might lead to over-fitting, which  jeopardizes  generalization. 

\subsection{Preliminaries}
Throughout the remainder we assume that the loss function $\ell_{\mathrm{f}}: \mbb{R}^{d\cdot D} \to \mbb{R}^+_0$ is continuous and non-negative. We require 
 \begin{equation} \label{eq:stat_minimizer}
\mbf{X}^\star(\mbf{y}) \doteq \argmin_{\mbf{x}\in \R^{d\cdot D}}~  \ell_{\mathrm{f}}(\mbf{x}, \mbf{y}) \not = \emptyset,
\end{equation}
i.e., the set of unconstrained minimizers of the loss function $\ell_{\mathrm{f}}$ is non-empty. 
Moreover, the set $\mbf{R}_N(\mbf{x}_0)\subseteq \R^{d\cdot D}$
\begin{equation} \label{eq:ReachableSet}
\mbf{R}_N(\mbf{x}_0) \doteq \left\{\mbf{x}\,|\, \mbf{x} = \mbf{x}_N(\mbf{x}_0), \|\mbf{u}_k\| < \infty,k\in\mbb{N}_{[0, N-1]}  \right\}
\end{equation}
collects all states reachable from $\mbf{x}_0$ within horizon $N$ by bounded control sequences $\{\mbf{u}_k\}$. Similarly, the set $\mbf{R}^\star_N(\mbf{x}_0)\subset \mbf{R}_N(\mbf{x}_0)$ denotes the set of states reached by optimal control sequences. 
Similar to \cite{Esteve20} we initially assume the following:
\begin{assumption}[Zero-loss DNN]\label{ass:doable}
Given data $\mbb{D}= \{\mbf{x}_0, \mbf{y}\}$ and $N\in\mbb{N}$ it holds that
$
\mbf{R}_N(\mbf{x}_0) \cap \mbf{X}^\star(\mbf{y}) \not = \emptyset$.  $\eDef$
\end{assumption}
The previous assumption can be understood as a realizability assumption, i.e., for sufficiently deep networks zero loss can be attained for the given data $\mbb{D}$. 

If the reachability assumption above holds, and assuming with only minor loss of generality\footnote{Indeed as $\ell_\mathrm{f}$ is non-negative, we can always consider $\ell_\mathrm{f}(\mbf{x}, \mbf{y}) - \ell_\mathrm{f}(\bar{\mbf{x}}, \mbf{y})$ with $\bar{\mbf{x}} \in  \mbf{X}^\star(\mbf{y})$ and $\ell_\mathrm{f}(\bar{\mbf{x}}, \mbf{y}) < \infty$. } that  
\begin{equation}\label{eq:0lossStat}
\mbf{x}\in \mbf{X}^\star(\mbf{y}) \quad \Leftrightarrow  \quad \ell_\mathrm{f}(\mbf{x}, \mbf{y}) = 0,
\end{equation}
we can rewrite the problem of attaining zero loss in \eqref{eq:OCP} as follows:
\begin{subequations}\label{eq:OCPcon}
\begin{align}
V_N(\mbf{x}_0) \doteq \min_{\{\mbf{u}_k\}}& \sum_{k=0}^{N-1} \ell(\mbf{x}_k, \mbf{u}_k) \\
\text{subject to}&~ \forall k \in \mbb{N}_{[0, N-1]}\nonumber \\
\mbf{x}_{k+1} &= f(\mbf{x}_k, \mbf{u}_k), \quad \mbf{x}_0 = \mathrm{vect}(\mbb{X})\subset \mbb{R}^{d\cdot D},\\
0&=\ell_\mathrm{f}(\mbf{x}_N, \mbf{y}), \label{eq:OCPcon_con}
\end{align}
\end{subequations}
wherein the loss function is replaced by the terminal equality constraint $0=\ell_\mathrm{f}(\mbf{x}_N, \mbf{y})$.

\begin{proposition}[Exact penalization of losses] \label{prop:0loss}
Suppose that Assumptions \ref{ass:sigma} and \ref{ass:doable} hold, and let the optimal pairs $(\mbf{x}^*_k,\,\mbf{u}^*_k)$  of OCP \eqref{eq:OCPcon} and $(\mbf{x}^\star_k,\,\mbf{u}^\star_k)$ of OCP \eqref{eq:OCP} satisfy second-order sufficient conditions.  

Then, for sufficiently large values of $\gamma \geq \gamma^\star$ and $N\in\mbb{N}$, the optimal state trajectory $\mbf{x}^\star_k, k\in \mbb{N}_{[0, N]}$ of   \eqref{eq:OCP} satisfies
$\mbf{x}^\star_N \in \mbf{X}^\star(\mbf{y})$, respectively,  $\ell_\mathrm{f}(\mbf{x}^\star_N, \mbf{y})=0$.  $\eDef$
\end{proposition}
\begin{proof}
Observe that $\ell_{\mathrm{f}}: \mbb{R}^{d\cdot D} \to \mbb{R}^+_0$ is scalar and non-negative, hence $\ell_\mathrm{f}(\mbf{x}_N, \mbf{y}) = |\ell_\mathrm{f}(\mbf{x}_N, \mbf{y})|$. Consequently, OCP \eqref{eq:OCP} can be considered as an exact penalty reformulation of OCP \eqref{eq:OCPcon}. 
Let $\mu^*$ be the Lagrange multiplier of \eqref{eq:OCPcon_con}  for the optimal solution $(\mbf{x}^*_k,\,\mbf{u}^*_k)$  of OCP \eqref{eq:OCPcon}. Applying \cite[Thm. 14.3.1]{Fletcher13a} gives that for  $\gamma \geq \gamma^\star = |\mu^\star|$ the assertion holds. 
\end{proof}
Though conceptually interesting, the above result has a number of pitfalls: 
(a) Proposition \ref{prop:0loss} relies on second-order sufficient optimality conditions which may not hold for arbitrary loss functions and which induce differentiability requirements. (b)  Moreover, Proposition \ref{prop:0loss} does not provide an explicit estimate for the required depth $N$. (c) Finally, the exact reachability condition of Assumption \ref{ass:doable} is quite strong and in general difficult to check. It may even be violated in some cases \cite{Steinberger00}.

\section{A Dissipativity Approach to DNN Design}
Consequently, we turn toward an alternative approach, which will not require exact reachability, while it provides an explicit depth estimate.
To this end, we recall a definition of dissipativity of OCPs, which can be traced back to \cite{Angeli12a}, while the notion of dissipative dynamical systems was coined by Jan Willems \cite{Willems72a}.

\begin{definition}[Strict dissipativity]\label{def:DI}~\\
\begin{enumerate}
\vspace*{-.5cm}
\item System \eqref{eq:stack_sys} is said to be \emph{dissipative with respect to a steady-state pair $ (\bar{\mbf{x}},\,\bar{\mbf{u}})$}, 
if there exists a non-negative function $ \lambda:\mbb{X} \to \mbb{R}^+_0$ such that for all $({\mbf{x}},\,\mbf{u})$
\begin{subequations} \label{eq:DI}
\begin{equation}\label{eq:DI_non_str}
 \lambda(f(\mbf{x}, \mbf{u})) - \lambda(\mbf{x}) \leq  \ell(\mbf{x}, \mbf{u}) - \ell(\bar{\mbf{x}}, \bar{\mbf{u}}).
\end{equation}
\item If, additionally, there exists $\alpha_\ell\in\mcl{K}$ such that
\begin{multline}  \label{eq:DI_str}
 \lambda(f(\mbf{x}, \mbf{u})) - \lambda(\mbf{x}) \leq \\-\alpha_\ell\left(\left\|(\mbf{x}, \mbf{u})-(\bar{\mbf{x}}, \bar{\mbf{u}})\right\|\right) + \ell(\mbf{x}, \mbf{u}) - \ell(\bar{\mbf{x}}, \bar{\mbf{u}}).
\end{multline}
\end{subequations}
then  \eqref{eq:stack_sys} is said to be \emph{strictly $x-u$ dissipative with respect to}  $ (\bar{\mbf{x}},\,\bar{\mbf{u}})$. 

\item If, for all $N\in\mbb{N}$ and all $\mbf{x}_0 \in \mbf{X}_0$, the dissipation inequalities \eqref{eq:DI} hold along any optimal pair of  \eqref{eq:OCP}, 
then \emph{OCP \eqref{eq:OCP} is said to be (strictly) $x-u$ dissipative with respect to  $ (\bar{\mbf{x}},\,\bar{\mbf{u}})$. }
\item Moreover, if 2) or 3) hold with $\alpha_\ell\left(\left\|(\mbf{x}, \mbf{u})-(\bar{\mbf{x}}, \bar{\mbf{u}})\right\|\right)$ replaced by $\alpha_\ell\left(\left\|\mbf{x}-\bar{\mbf{x}}\right\|\right)$, then system \eqref{eq:stack_sys}, respectively, OCP \eqref{eq:OCP} are said to be strictly $x$ dissipative with respect to  $ (\bar{\mbf{x}},\,\bar{\mbf{u}})$. $\eDef$    
\end{enumerate}
\end{definition}
Observe the fact that for the dynamics \eqref{eq:stack_sys} any state $\bar{\mbf{x}}$ constitutes a controlled equilibrium with corresponding $\bar{\mbf{u}} = 0$. Moreover, note that the dissipativity of OCP \eqref{eq:OCP} only depends on the regularization $\ell$ and not on the loss function $\ell_\mathrm{f}$.

\subsection{Turnpikes in DNN Training} \label{sec:results}
Consider
\begin{align*}
\mcl{Q}_\varepsilon &\doteq  \left\{k \in \mbb{N}_{[0, N-1]}\,|\, \left\|\mbf{x}^\star_k-\bar{\mbf{x}}\right\| \leq \varepsilon\right\} \quad
\widehat{\mcl{Q}}_\varepsilon  \doteq \mbb{N}_{[0, N-1]} \setminus \mcl{Q}_\varepsilon.
\end{align*}

\begin{assumption}[Strict dissipativity of OCP \eqref{eq:OCP}]\label{ass:DI}
For the given data $\mbb{D}= \{\mbf{x}_0, \mbf{y}\}$, with $\mbf{x}_0\in\mbf{X}_0$, OCP \eqref{eq:OCP} is strictly $x$ dissipative with respect to $(\bar{\mbf{x}}, \bar{\mbf{u}})$.  $\eDef$
\end{assumption}
While Assumption \ref{ass:doable}  requires finite-time reachability of some unconstrained minimizer of the loss function, the next assumption defines an exponential reachability property with respect to a specific equilibrium  $(\bar{\mbf{x}}, \bar{\mbf{u}})$. To this end, we distinguish  the  $x$-projection of the data $\mbb{X}$---which corresponds to the initial condition $\mbf{x}_0$ in the OCPs \eqref{eq:OCP} and \eqref{eq:OCPcon}---from the set of possible initial conditions $\mbf{X}_0 \subseteq \mbb{R}^{d\cdot D}$ of the ensemble dynamics \eqref{eq:stack_sys}. Put differently, we have that $\mathrm{vect}(\mbb{X}) \in \mbf{X}_0 \subseteq \mbb{R}^{d\cdot D}$.
\begin{assumption}[Exponential reachability]\label{ass:ExpReach}
 There exist constants $\rho \in [0, 1)$ and $\beta >0$, an infinite-horizon control $\tilde{\mbf{u}}: \mbb{N}_{[0, \infty)} \to \mbb{R}^{d^2 +d}$, and a class $\mcl{K}$ function $\hat\alpha_\ell:\mbb{R}_0^+ \to \mbb{R}_0^+$ such that, for all $\mbf{x}_0\in\mbf{X}_0$,
\[
\hat\alpha_\ell(\left\|(\tilde{\mbf{x}}_k, \tilde{\mbf{u}}_k)-(\bar{\mbf{x}}, \bar{\mbf{u}})\right\|) \leq \beta\rho^k
\]
and $\ell(\mbf{x}, \mbf{u}) \leq \hat\alpha_\ell(\left\|(\mbf{x}, \mbf{u})-(\bar{\mbf{x}}, \bar{\mbf{u}})\right\|)$ on $\mbb{R}^{D} \times \mbb{R}^{d^2+d}$.
 $\eDef$
\end{assumption}
First we analyze the structure of optimal solutions implied by strict dissipativity, i.e., the next result establishes a turnpike property in OCP \eqref{eq:OCP}.

\begin{proposition}[Turnpikes in DNN training]\label{prop:TP}
Suppose that Assumptions \ref{ass:DI} -- \ref{ass:ExpReach} hold. Moreover, suppose that, for all $\mbf{x} \in \mbf{R}^\star_N(\mbf{x}_0)$, the storage function $\lambda$ is bounded from below. Let $\lambda$ be bounded on $\mbf{X}_0$.  Then, there exist constants $\Lambda, \hat V >0$ such that, for any $\gamma \in\mbb{R}$
\begin{itemize}
\item $\# \mcl{Q}_\varepsilon \geq N - \dfrac{\Lambda +\hat V}{\alpha_\ell(\varepsilon)}$, respectively, $\# \widehat{\mcl{Q}}_\varepsilon \leq \dfrac{\Lambda +\hat V}{\alpha_\ell(\varepsilon)}$,
\end{itemize}
where $\# \mcl{Q}_\varepsilon $ is the cardinality of the set $\mcl{Q}_\varepsilon$.  $\eDef$
\end{proposition}
\begin{proof}
Without loss of generality, we set $\ell(\bar{\mbf{x}}, \bar{\mbf{u}}) = 0$.
Moreover, Assumption \ref{ass:ExpReach}  gives that there exists $\hat V \in \mbb{R}$
$V_N^{\gamma}(\mbf{x}_0) \leq \hat V$.
The strict dissipation inequality \eqref{eq:DI_str} implies that
\[
\lambda(\mbf{x}_N^\star)) - \lambda(\mbf{x}_0) \leq -\sum_{k=0}^{N-1} \alpha_\ell(\left\|\mbf{x}^\star_k-\bar{\mbf{x}}\right\|) + \ell(\mbf{x}^\star_k, \mbf{u}^\star_k).
\]
Hence 
\begin{multline*}
 \lambda(\mbf{x}^\star_N) - \lambda(\mbf{x}_0) +\sum_{k=0}^{N-1} \alpha_\ell(\left\|\mbf{x}^\star_k-\bar{\mbf{x}}\right\|)  \leq V_N^{\gamma}(\mbf{x}_0) -  \gamma\ell_\mathrm{f}(\mbf{x}^\star_N, \mbf{y}).
\end{multline*}
Observe that due the assumption that the storage is bounded from below on $\mbf{R}_N^\star(\mbf{x}_0)$, and from above on $\mbf{X}_0$. Hence  there exists $\Lambda \in \mbb{R}$, 
$
\lambda(\mbf{x}^\star_N) - \lambda(\mbf{x}_0) \geq -\Lambda$.
From the above, we obtain
\[
(N-\#\mcl{Q}_\varepsilon)\alpha_\ell(\varepsilon) \leq \sum_{k=0}^{N-1} \alpha_\ell(\left\|\mbf{x}^\star_k-\bar{\mbf{x}}\right\|)\leq \Lambda + \hat V.
\]
Rearranging gives
$\#\mcl{Q}_\varepsilon \geq N - \frac{\Lambda + \hat V}{\alpha_\ell(\varepsilon)}$ and
$\# \widehat{\mcl{Q}}_\varepsilon \leq \frac{\Lambda + \hat V}{\alpha_\ell(\varepsilon)}$.
\end{proof}
It is worth to be remarked that the result above establishes a turnpike property not only for $\mbf{x}_0 = \mathrm{vect}(\mbb{X})$ but indeed on a set of initial conditions $\mbf{X}_0$. This in turn implies that small pertubations of the data $\mbb{D}$ should not affect the statement (provided Assumption~\ref{ass:ExpReach} holds).
So far, our analysis has not made any explicit assumption on the structure of $\ell$---besides Assumption \ref{ass:ExpReach}---which could easily be re-formulated without reference to $\ell$. Moreover, the turnpike property established in the previous result does not depend crucially on the considered loss function $\ell_\mathrm{f}$ but mostly on the regularization stage-cost $\ell$. Hence it makes sense to choose $\ell$ having the considered loss function in mind. 
Subsequently, we suppose that 
\begin{equation} \label{eq:LQell}
\ell(\mbf{x}, \mbf{u}) = r\|\mbf{u}\|^{s_x}_{p_x} + q\|\mbf{x}-\bar{\mbf{x}}\|^{s_u}_{p_u}, \quad \bar{\mbf{x}} \in \mbf{X}^\star(\mbf{y}_0), 
\end{equation}
with $q,r >0$. We summarize the parameters of the stage cost $\ell$ by writing $\pi \doteq (s_x, s_u, p_x, p_u, q, r)$. Without loss of generality, we temporarily set $s_x =s_u =2$, $p_x =p_u =2$ for the further analysis. As we will see in Proposition \ref{prop:infnorm} one may want to consider $p_x = \infty$ to reduce the dependence of the actual bounds on the cardinality of the data set $\mbb{D}$. 

That is, we consider a strictly convex regularization with respect to some unconstrained minimizer of $\ell_\mathrm{f}$. 
\begin{lemma}[Strict dissipativity of OCP \eqref{eq:OCP}]\label{lem:DI}
Suppose that the regularization stage costs satisfies \eqref{eq:LQell} with $q,r>1$, that  Assumption \ref{ass:ExpReach} holds with respect to $(\bar{\mbf{x}}, \mbf{0})$ penalized in $\ell$. Then system \eqref{eq:stack_sys} and  OCP \eqref{eq:OCP} are strictly $x-u$ dissipative with respect to $(\bar{\mbf{x}}, \mbf{0})$. Moreover, $\lambda(\mbf{x}) \equiv 0$ and
\[\qquad \alpha_\ell(\left\|(\mbf{x}, \mbf{u})-(\bar{\mbf{x}}, \bar{\mbf{u}})\right\|) =  \nu\ell(\mbf{x}, \mbf{u}),\quad  \nu \in (0,1]. \qquad \quad \eDef
\]
 
\end{lemma}
The proof follows directly from the available storage characterization of dissipativity \cite{Willems72a}. It is thus omitted. 

\subsection{Depth Bounds for DNNs}
While Proposition \ref{prop:0loss} leverages an exact penalty function and a finite time reachability condition to ensure zero loss, it fails to provide an estimate on the required depth of the network under weaker (i.e. asymptotic) assumptions. The next result addresses this gap, i.e., it presents a quantitative bound on the network depth ensuring $\varepsilon$-loss. Let $\mcl{N}_\varepsilon(\bar{\mbf{x}})$ denote an $\varepsilon$-neighborhood of $\bar{\mbf{x}}$.

\begin{theorem}[Upper bound on required DNN depth $N$] \label{thm:epsloss}
Let Assumptions \ref{ass:sigma} and \ref{ass:ExpReach} hold. Suppose that the mixed input-state regularization \eqref{eq:LQell} is used with $s_x =s_u =2 = p_x =p_u$ and $q,r>1$, and that the loss function $\ell_\mathrm{f}$ is locally Lipschitz on $\mcl{N}_\varepsilon(\bar{\mbf{x}})$ with constant $L_{\ell_\mathrm{f}}$. 
Then, for all $\gamma > 0$ and all $\varepsilon >0$, choosing
the network depth  in OCP \eqref{eq:OCP} as
\begin{equation}\label{eq:N}
N \geq \hat{N}(\varepsilon) \doteq \dfrac{\hat V}{\alpha_\ell(\varepsilon)} = \dfrac{\beta}{(1-\rho)q\varepsilon^2}
\end{equation}
 with $\beta,\rho\geq 0$ from Assumption \ref{ass:ExpReach} guarantees that 
\[\ell_\mathrm{f}(\mbf{x}^\star_N,\mbf{y}) 
\leq \dfrac{L_{\ell_\mathrm{f}}}{\gamma} \varepsilon. \] 
Moreover,  we have that $\textrm{dist}(\mbf{x}^\star_N, \mbf{X}^\star(\mbf{y})) \leq \varepsilon$.$\eDef$
\end{theorem}
\begin{proof}
First note that due to Assumption \ref{ass:ExpReach} and the choice of the regularizing stage cost \eqref{eq:LQell}, we have that the constant $\hat V$ in Proposition \ref{prop:TP} is given by
$\hat V = \frac{\beta}{1-\rho} > 0$.
Moreover, Lemma \ref{lem:DI} gives that 
\[
\alpha_\ell(\left\|\mbf{x}-\bar{\mbf{x}})\right\|) \leq \ell(\mbf{x}, \mbf{u}) =\hat\alpha_\ell(\left\|(\mbf{x}, \mbf{u})-(\bar{\mbf{x}}, \bar{\mbf{u}})\right\|) 
\]
and $\Lambda = 0$.
As $\bar{\mbf{u}} = 0$ and $s_x=2$ we have $\alpha_\ell(\varepsilon) = q\varepsilon^2$. Hence the bound of Proposition \ref{prop:TP} equates to 
$\# \widehat{\mcl{Q}}_\varepsilon \leq \frac{\hat V}{\alpha_\ell(\varepsilon)} = 
 \frac{\beta}{(1-\rho)q\varepsilon^2}$.
In other words, the time the optimal solutions can spend outside of the $\varepsilon$-neighborhood of $\bar{\mbf{x}}$ is bounded from above by $\frac{\beta}{(1-\rho)q\varepsilon^2}$. Hence, for sufficiently large $N$, there exists $l \in \mbb{N}_{[0, N-1]}$ for which 
$
\left\|\mbf{x}^\star_{l}-\bar{\mbf{x}}\right\| \leq \varepsilon$.
Moreover, observe that if $\mbf{x}_N^\star \in \mcl{N}_\varepsilon(\bar{\mbf{x}})$,  then the estimate $\ell_\mathrm{f}(\mbf{x}^\star_N,\mbf{y}) \leq \frac{L_{\ell_\mathrm{f}}}{\gamma} \varepsilon$ follows immediately. Therefore, we now show $\mbf{x}_N^\star \in \mcl{N}_\varepsilon(\bar{\mbf{x}})$ for sufficiently large $N$.

As the entire state space $\R^{d\cdot D}$ corresponds to equilibria of \eqref{eq:stack_sys}, for any $\varepsilon \geq 0$, the choice of $A_k = 0, b_k =0$ renders $\mcl{N}_\varepsilon(\bar{\mbf{x}})$ forward invariant under \eqref{eq:stack_sys}.  The performance associated to this choice implies the following upper bound on the performance on the truncated horizon $\mbb{N}_{[l,N]}$
\begin{equation} \label{eq:bnd1}
\sum_{k=l}^{N-1} \ell(\mbf{x}^\star_{k}, \mbf{u}^\star_{k}) + \gamma\ell_\mathrm{f}(\mbf{x}^\star_{N}, \mbf{y}) \leq (N-l)q\varepsilon^2 + \gamma L_{\ell_\mathrm{f}}\varepsilon,
\end{equation}
provided that $\ell_\mathrm{f}$ is Lipschitz on $\mcl{N}_\varepsilon(\bar{\mbf{x}})$. By assumption, if $\varepsilon$ is small enough this will be the case.

Now, we distinguish three cases:
\begin{itemize}
\item[(i)] the optimal trajectory $\mbf{x}^\star_k$ leaves  $\mcl{N}_\varepsilon(\bar{\mbf{x}})$;
\item[(ii)]  $\mbf{x}^\star_k$ leaves and re-enters  $\mcl{N}_\varepsilon(\bar{\mbf{x}})$;
\item[(iii)] $\mbf{x}^\star_k$ remains in  $\mcl{N}_\varepsilon(\bar{\mbf{x}})$.
\end{itemize}
Observe that cases of leave-enter-leave can be reduced to enter-leave by considering the final exit time. 
Moreover, Case (iii) corresponds to the assertion and does not require further analysis.

Case (i): At time $m\geq l$, the optimal trajectory $\mbf{x}^\star_k$ leaves  $\mcl{N}_\varepsilon(\bar{\mbf{x}})$, i.e.,
$\left\|\mbf{x}^\star_{k}-\bar{\mbf{x}}\right\| > \varepsilon \quad \forall k \geq m$.
This implies the following performance bound
$
\sum_{k=m}^{N-1} \ell(\mbf{x}^\star_{k}, \mbf{u}^\star_{k}) + \gamma\ell_\mathrm{f}(\mbf{x}^\star_{N}, \mbf{y}) > (N-m)q\varepsilon^2 + \gamma L_{\ell_\mathrm{f}}\varepsilon$,
which in turn contradicts the bound from \eqref{eq:bnd1} for $l=m$. Hence,  solutions permanently leaving  $\mcl{N}_\varepsilon(\bar{\mbf{x}})$ are suboptimal.

Case (ii): At time $m_1\geq l$, the optimal trajectory $\mbf{x}^\star_k$ leaves  $\mcl{N}_\varepsilon(\bar{\mbf{x}})$ and at  $m_2= m_1+\Delta m > m_1 \geq l$ it re-enters. This implies the lower bound on $\mbb{N}_{[m_1,N]}$
\begin{multline} \label{eq:bnd2}
\sum_{k=m}^{N-1} \ell(\mbf{x}^\star_{k}, \mbf{u}^\star_{k}) 
\geq  \Delta m q(\varepsilon+\eta)^2 + 
\sum_{k=m_2}^{N-1} \ell(\mbf{x}^\star_{k}, \mbf{u}^\star_{k}) 
\end{multline}
where $\eta >0$ corresponds to the distance of $\mbf{x}^\star_{k}$, $k \in \mbb{N}_{[m_1, m_2)}$ to  $\mcl{N}_\varepsilon(\bar{\mbf{x}})$. For the l.h.s. of \eqref{eq:bnd2} to not exceed the r.h.s. of \eqref{eq:bnd1}, the optimal trajectory $\mbf{x}^\star_k$ has to enter  $\mcl{N}_\delta(\bar{\mbf{x}})$ with $\delta \leq \varepsilon -\eta$. So either, case (ii) does not happen, or the optimal trajectory $\mbf{x}^\star_k$ enters  $\mcl{N}_\delta(\bar{\mbf{x}})$ on $\mbb{N}_{[N-\tilde m, N]}$ for some $\tilde m \geq m_1+\Delta m$. 
Now, repeat the analysis for  $\mcl{N}_\delta(\bar{\mbf{x}})$. Induction shows that $\mbf{x}_N^\star(\mbf{x}_0) \in \mcl{N}_\varepsilon(\bar{\mbf{x}})$.

Finally, we have to show that $\textrm{dist}(\mbf{x}^\star_N, \mbf{X}^\star(\mbf{y})) \leq \varepsilon$. Above we have derived that $\mbf{x}_N^\star(\mbf{x}_0) \in \mcl{N}_\varepsilon(\bar{\mbf{x}})$. As $\bar{\mbf{x}} \in
\mbf{X}^\star(\mbf{y})$, we have that  $\|\mbf{x}_N^\star(\mbf{x}_0) - \bar{\mbf{x}}\| \geq \textrm{dist}(\mbf{x}^\star_N, \mbf{X}^\star(\mbf{y}))$. This concludes the proof. 
\end{proof}

Let $N^\star$ denote the minimal horizon length $N$ for which Assumption \ref{ass:doable} holds and recall that $\gamma^\star$ denotes the minimal value of $\gamma$ from Proposition~\ref{prop:0loss}, which ensures exact loss minimization. This allows to formulate the following corollary to Proposition \ref{prop:0loss} and Theorem \ref{thm:epsloss}.
\begin{corollary}[Zero loss with finite depth]\label{cor:exact}
Let Assumptions \ref{ass:sigma} -- \ref{ass:doable} hold, suppose that the mixed input-state regularization \eqref{eq:LQell} is used  with $s_x =s_u =2 = p_x =p_u$ and $q,r>1$, that the loss function $\varphi$ is locally Lipschitz on $\mcl{N}_\varepsilon(\bar{\mbf{x}})$, and that second-order sufficient conditions hold at $(\mbf{x}^\star_k, \mbf{u}^\star_k)$.
Provided $\gamma \geq \gamma^\star$ for $\gamma^\star$ from Proposition~\ref{prop:0loss}, that 
\[N\geq \hat{N}(\varepsilon) \doteq \dfrac{\beta}{(1-\rho)q\varepsilon^2}\geq N^\star
\] 
 holds in OCP \eqref{eq:OCP}, and that $\beta,\rho\geq 0$ from Assumption \ref{ass:ExpReach} are considered, then
$ \textrm{dist}(\mbf{x}^\star_N, \mbf{X}^\star(\mbf{y})) =0$. $\eDef$
\end{corollary}
The proof follows directly from combining Proposition \ref{prop:0loss} and Theorem \ref{thm:epsloss}. The result shows that if the horizon bound from Theorem \ref{thm:epsloss} is combined with a large value of $\gamma$, then one may even achieve zero loss. 

\subsection{Bounds Independent of the Number of Samples $\# \mbb{D}$}
At this point, it is fair to ask how to specify the remaining degrees of freedom in the regularization $\ell$ from \eqref{eq:LQell}? It is straighforward to see that the choice $s_x =s_u =2$  is not crucial in the proofs of the results above. On other hand, from a numerical point of view, it is promising to  use $p_x =p_u = s_x =s_u =2$ and $q\gg r=1$, i.e., a convex quadratic regularization stage cost. Yet, employing the two-norm means that the bound $\hat{N}(\varepsilon)$ will scale  with $\dim(\mbf{x}) = d\cdot D$, which scales with the number of consider data samples $D=\#\mbb{D}$. Observe that normalizing the stage cost $\ell$ with $D$ is no viable remedy as in this case $\alpha_\ell$ has to be normalized by $D$ as well. This suggests to consider $p_x  =\infty$ in $\ell$ from \eqref{eq:elli_class} as this choice does not scale with $D=\#\mbb{D}$. However, optimizing over $\infty$-norm objectives with non-linear equality constraints  is not straightforward.

In this context, the next result shows how one may  compute an a-posteriori estimate $\hat{N}$ using the  $\infty$-norm, while the training considers a numerically more favorable norm. To this end, let $\ell(\mbf{x},\mbf{u}; \pi)$ denote the stage cost \eqref{eq:LQell} with parameters $\pi =(s_x, s_u, p_x, p_u, q, r)$. Any choice with $q,r >0$, $p_x, p_u >1$ and $s_x, s_u >1$ is said to be admissible. Moreover, let $\hat{N}(\varepsilon; \pi)$ and $V_N^\gamma(\mbf{x_0}; \pi)$ denote the respective dependence on $\pi$.

\begin{proposition}[Transferring depth 	bounds between norms]\label{prop:infnorm}
Let $(\tilde{\mbf{x}}_k, \tilde{\mbf{u}}_k)$ be an optimal solution of OCP \eqref{eq:OCP} with a mixed input-state regularization $\ell(\mbf{x},\mbf{u}; \tilde\pi)$ from \eqref{eq:LQell}. Except for the choice of $\tilde \pi$, let the conditions of Corollary \ref{cor:exact} hold. Then, for any admissible parametrization $\pi=(s_x, s_u, p_x, p_u, q, r)$, the horizon
\begin{equation} \label{eq:NfromJ}
N \geq \hat{N}(\varepsilon;\pi) \doteq \frac{1}{q\cdot\varepsilon^{s_x}} \sum_{k=0}^{\hat{N}(\varepsilon;\tilde \pi)}\ell(\hat{\mbf{x}}_k,\,\hat{\mbf{u}}_k; \pi)\geq N^\star
\end{equation} 
in OCP \eqref{eq:OCP} ensures  $\ell_\mathrm{f}(\mbf{x}^\star_N,\mbf{y}) = 0$.  $\eDef$
\end{proposition}
\begin{proof}
Observe that the depth bound can be written as $\hat{N}(\varepsilon) \doteq \frac{\hat V}{\alpha_\ell(\varepsilon)}$, i.e., as the quotient of an upper bound on the value function $V_N^\gamma(\mbf{x_0}; \pi)$ with the lower bound on $\ell$. The conditions of Corollary~\ref{cor:exact} imply that $\ell_\mathrm{f}(\hat{\mbf{x}}_N,\mbf{y}) = 0$ and hence, for any admissible stage cost parametrization $\pi$, the sum  
$
\tilde V \doteq \sum_{k=0}^{\hat{N}(\varepsilon;\tilde \pi)}\ell(\hat{\mbf{x}}_k,\,\hat{\mbf{u}}_k; \pi)
$
 gives an upper bound on the value function $V_N^\gamma(\mbf{x_0}; \pi)$. Moreover, for any admissible choice  $\pi$, we have $\alpha_\ell(\varepsilon) = q\|\varepsilon\|_{p_x}^{s_x} =  q\cdot\varepsilon^{s_x}$.
\end{proof}
While at the first glance the above result is of merely technical value, it enables to use numerically favourable choices for the parametrization of $\ell$ for training purposes, while the depth bounds can rely on the $\infty$-norm for the state regularization, which does not depend on the cardinality of $\mbb{D}$. 

\subsection{Empirical Risk Minimization and the Choice of $\varepsilon$}
In principle, the choice $\varepsilon >0$ may lead to a slight performance degradation in terms of loss minimization. However, a value of $\varepsilon$ too small might render the bound of \eqref{eq:N} conservative as $\hat{N}=\infty$ for $\varepsilon \to 0$. Hence, depending on the considered ML task at hand (classification or regression), it may be advisable to choose $\varepsilon$ not too small. In essence, in classification problems the choice of $\varepsilon$ is governed by the size of the neighborhood of $\bar{\mbf{x}}\in\mbf{X}^\star(\mbf{y})$ which still allows to classify the data points. 
To elaborate this, let $g:\mbb{R}^d \to \mbb{R}^m$, $y = g(x)$ 
denote map from the $N$th layer of \eqref{eq:sys} to the label prediction $y$. That is, $y^\star(x_0 ^i) = g(x_N^\star(x_0^i))$ is the propagation of the data point $x_0^i$ through the DNN \eqref{eq:sys}, using optimal weights obtained via OCP~\eqref{eq:OCP}, concatenated with $g$.

Suppose, as before, that \eqref{eq:0lossStat} holds and let $\bar{\mbf{x}} = \begin{pmatrix}\bar x^{1\top},\dots,\bar x^{i\top}, \dots, \bar x^{D\top}
\end{pmatrix}^\top$, then  $g$ may be defined as follows
\begin{equation}\label{eq:gdelta}
g(x^i) \doteq \left\{\begin{array}{l l}
\| x^i - \bar x^i \| < \delta & y = y^i \\
\| x^i - \bar x^i \| \geq \delta & y \not = y^i
\end{array} \right.,
\end{equation}
where $y^i \in \mbb{Y}$ is a finite set of classification labels and $\delta>0$ is the maximal radius of $\mcl{N}_\delta(\bar x^i)$ which still allows to exactly distinguish the underlying classes for the entire data set $\mbb{D}$ (provided it exists).
This choice of $g$ suggests the loss function for the $i$th data sample to be
$\ell^i_{\mathrm{f}}(x_N(x^i),y^i) = \| g(x_N^\star(x_0^i)) - y^i \|^s$,
which is not differentiable on $\| x^i - \bar x^i \| = \delta$ due to \eqref{eq:gdelta}.
In view of Theorem \ref{thm:epsloss}---$\textrm{dist}(\mbf{x}^\star_N, \mbf{X}^\star(\mbf{y})) \leq \varepsilon$---one may consider the following differentiable substitute
\begin{equation}\label{eq:elli_class}
\ell^i_{\mathrm{f}}(x_N(x_0^i),y^i) = \| x_N^\star(x_0^i) - \bar x^i\|^2
\end{equation}
which reformulates the loss in terms of the squared distance to $\bar x^i$. As in the proof of Theorem \ref{thm:epsloss} we have shown that $\mbf{x}_N^\star(\mbf{x}_0) \in \mcl{N}_\varepsilon(\bar{\mbf{x}})$, it is  clear that choosing $\varepsilon < \delta$ will lead to exact classification of all sample points in $\mbb{D}$.

Consider the index set
\[
\mcl{I}_N(\mbb{D}) \doteq  \{i \in \{1,\cdots, D\} \,|\, g(x_N^\star(x_0^i))  \not= y^i\}
\]
of misclassified samples from $\mbb{D}$. 
Then the empirical risk (of misclassification) is defined as
\[\mcl{R}_N(\mbb{D}) \doteq \dfrac{\#\mcl{I}_N(\mbb{D}) }{\#\mbb{D}},
\]
see \cite{Shalev14}.
The next result translates Theorem \ref{thm:epsloss} to zero empirical risk classification. 
\begin{proposition}[Zero empirical risk classification] \label{prop:0lossClass}
Let Assumptions \ref{ass:sigma} and \ref{ass:ExpReach} hold. Suppose that the mixed input-state regularization \eqref{eq:LQell} is used and consider $g$ from \eqref{eq:gdelta} and $\ell_\mathrm{f}^i$ from \eqref{eq:elli_class}. Suppose that there exists $\delta >0$, which provides perfect classification. 
Then, if $N > \hat{N}(\varepsilon=\delta)$, it holds that $\mcl{R}_N(\mbb{D})  = 0$. $\eDef$
\end{proposition}
The proof follows directly from the considerations above and is thus omitted. 

Naturally, if the depth bounds $\hat{N}$ in \eqref{eq:N} are too conservative, the potential price to pay for zero empirical risk classification as per Proposition~\ref{prop:0lossClass} are insufficient generalization properties---i.e., a larger (true) risk of misclassification---due to over-fitting. Thus, we turn towards  a numerical example to analyze this aspect.

\section{Numerical Example -- Two Spiral Task}\label{sec:example}
\begin{figure*}[t!]
    \centering
    \begin{subfigure}[t]{0.3\textwidth}
	\includegraphics[scale=0.25]{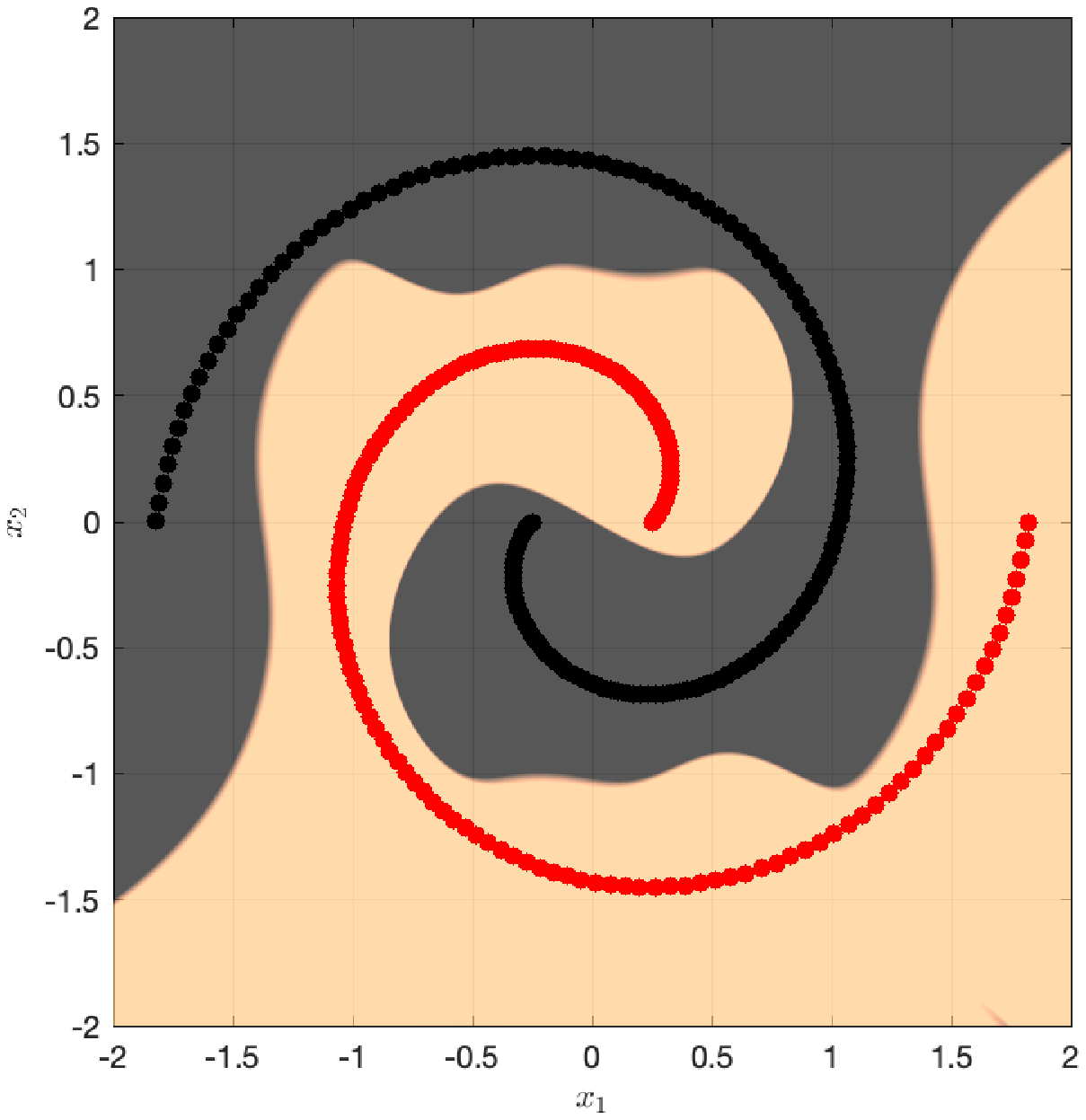}~
        \caption{$D=300$, no noise, and  $N=5$.  \label{fig:xygen_k5_ideal}}
    \end{subfigure}%
                \begin{subfigure}[t]{0.3\textwidth}
	\includegraphics[scale=0.25]{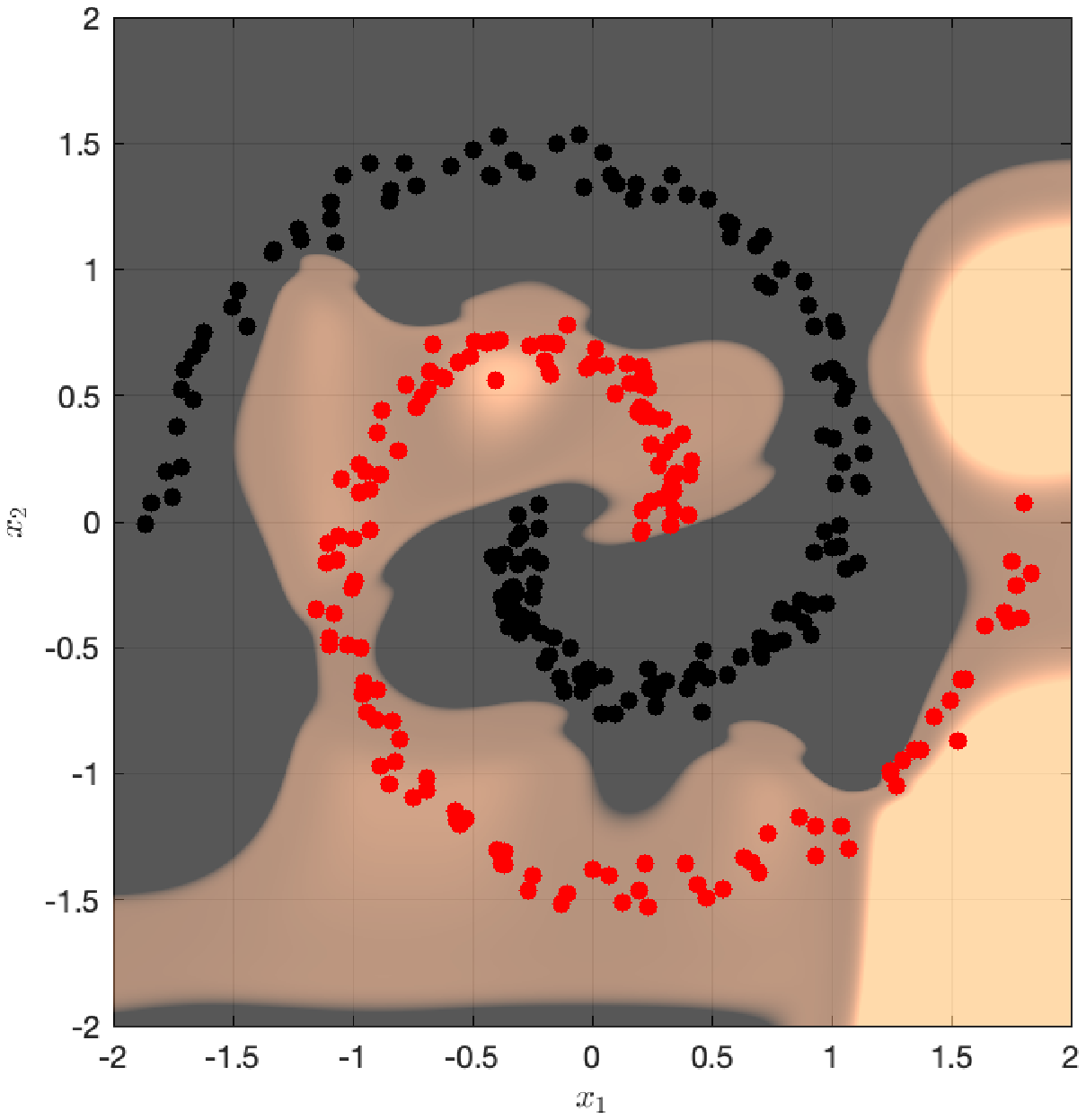}
        \caption{$D=300$, additive noise, and  $N=5$. \label{fig:xygen_k5}}
    \end{subfigure}%
        \begin{subfigure}[t]{0.3\textwidth}
        \centering
	\includegraphics[scale=0.25]{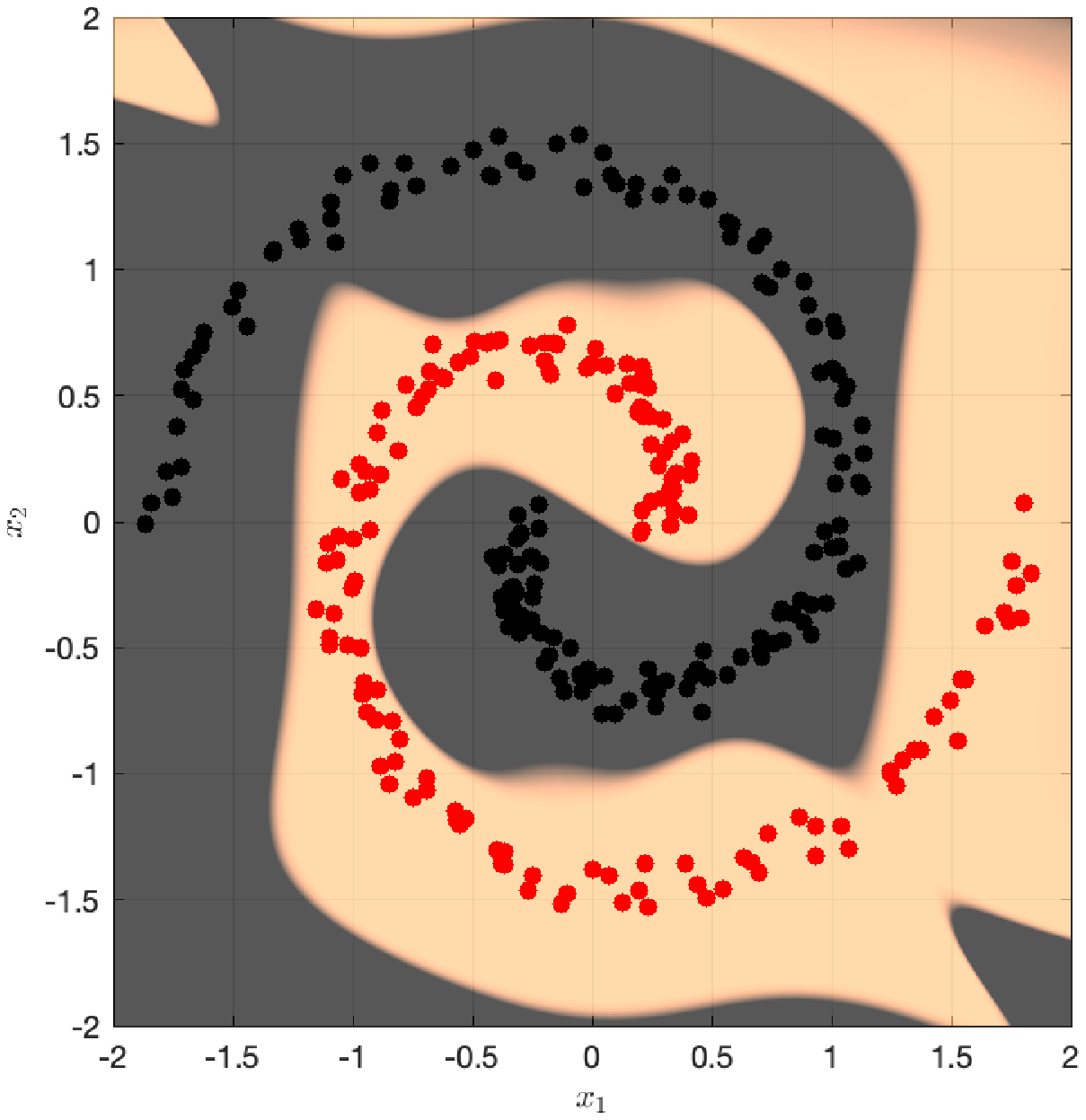}
       \caption{$D=300$, additive noise, and  $N=10$.  \label{fig:xygen}}
    \end{subfigure}
    ~ 
    \begin{subfigure}[t]{0.5\textwidth}
        \centering
        \includegraphics[scale=0.45]{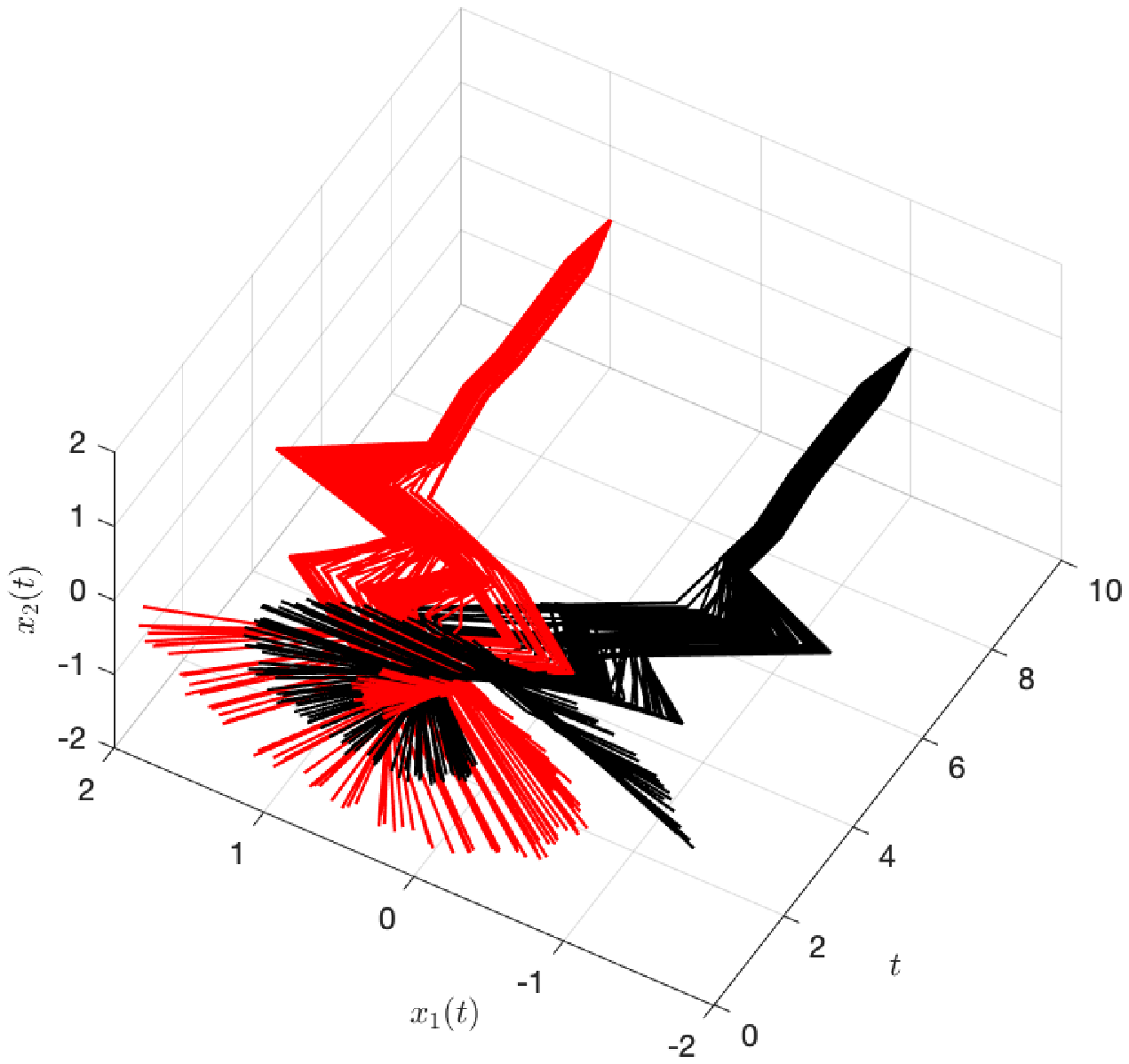}
        \caption{Evolution of data samples for $D=300$ with noise in \eqref{eq:data_gen}.
\label{fig:xy}}
    \end{subfigure}%
    \begin{subfigure}[t]{0.49\textwidth}
        \centering
        \includegraphics[scale=0.5]{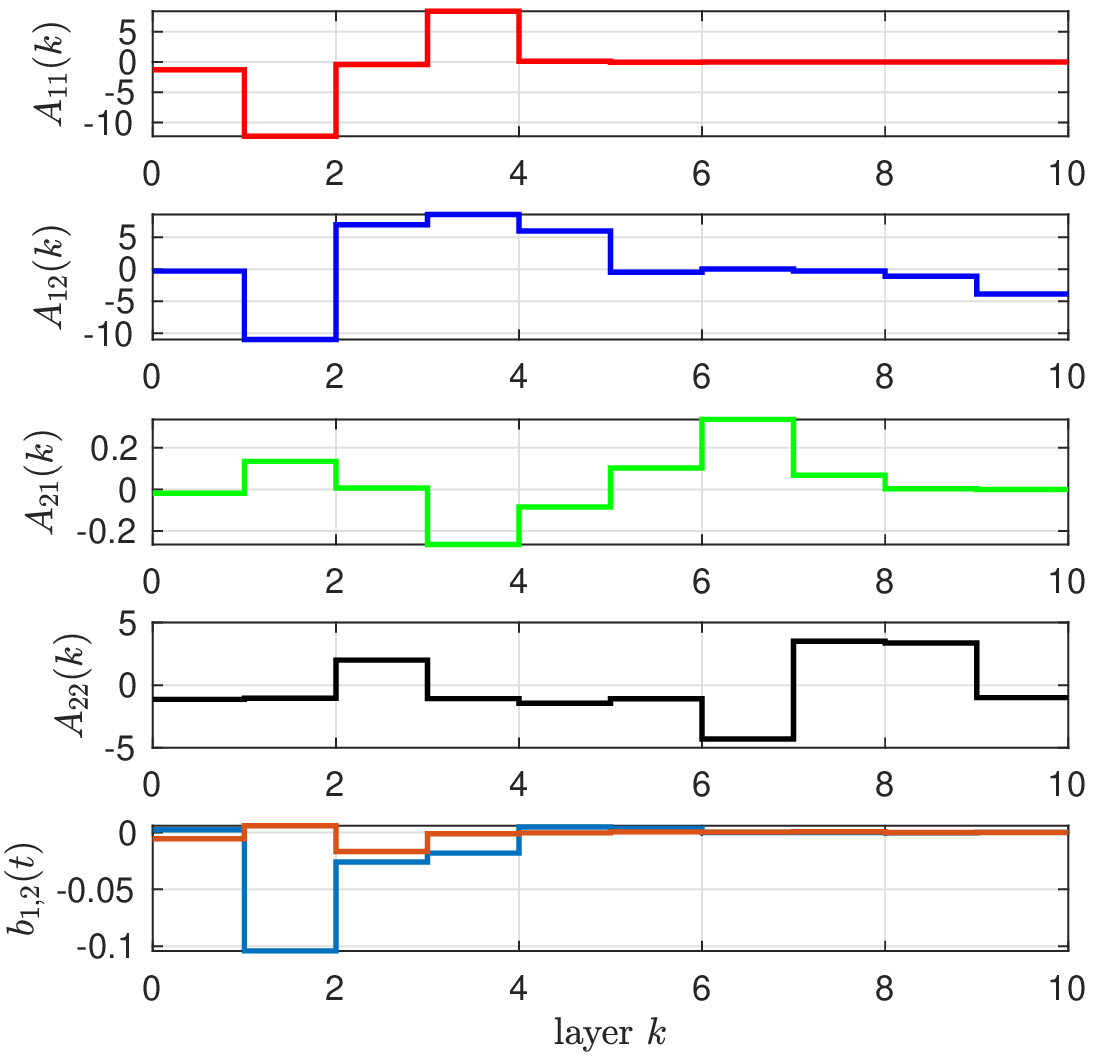}
        \caption{Evolution of the network weights for $D=300$ with noise in \eqref{eq:data_gen}. \label{fig:u}}
    \end{subfigure}
    \caption{Results for TST example.}
\end{figure*}
We consider the two dimensional Two Spiral Task (TST) data set. The TST problem refers to a binary classification of data points. As depicted in Figures \ref{fig:xygen_k5_ideal} -- \ref{fig:xygen}, the TST is build upon two intertwined spirals, which are not easy to classify in the given feature space. 
This allows to check whether the OCP-trained DNN is capable to represent complex decision boundaries (classification) or to learn underlying manifolds (regression). Since its first suggestion~\cite{Lang89} this problem has become a standard ML test case~\cite{Wah00,Chalup07}. 
The underlying manifold for the 2D learning data generates data samples via
\begin{align} \label{eq:data_gen}
{x}^{i}_{0} &= \begin{pmatrix} j^i \nu \cos\left(j^i+\varphi_{0c} \right) \\
							    j^i \nu \sin\left(j^i+\varphi_{0c} \right) \end{pmatrix}+x^i_{\text{noise}},
\end{align}
where $j$ increments angle and radius, $\nu=\frac{1}{4}$ is the TST start radius, and $\varphi_{0c}$ determines the initial angles. 
The increments are given by $j^i=\frac{2\pi}{D}i, i=1,\ldots,D$. The start angle $\varphi_{0c}$ is used for class assignment, i.e., $\varphi_{01}=0$ for the first class (red) and $\varphi_{02}=\pi$ for the second class (black), see Figure \ref{fig:xygen_k5_ideal}. 
Except the benchmark, each element of the learning data set ${x}^{i}\in \mbb{D}$ is superposed with uniformly distributed noise $x^i_{\text{noise}} \in \begin{bmatrix}-0.2 & 0.2\end{bmatrix}\times \begin{bmatrix}-0.2 & 0.2\end{bmatrix}$.
The goal of the classification is to map class one to $y^i = \begin{pmatrix}1&0\end{pmatrix}^\top$, and class two to $y^i=\begin{pmatrix}-1&0\end{pmatrix}^\top$. 

\subsection{OCP Formulation and Solution}
We consider OCP \eqref{eq:OCP} to train the DNN. 	
We use $\ell$ from \eqref{eq:LQell} with $\|\cdot\|^2_2$ for $\mbf{x}-\bar{\mbf{x}}$ and $\mbf{u}$, as well as $q = 10^{-2}$ and $r=10^{-3}$.\footnote{Actually, for sake of numerical stability, we rescale the inputs by factor $10^2$ such that in re-scaled input variables we have $r = 10^{-7}$.} Simulations are done using CasADi in Matlab (Version 3.55) \cite{Andersson19a} and IPOPT as NLP solver. The considered loss function is of mean-squared structure \eqref{eq:elli_class},
where the reference state for regularization coincides with the labels, i.e., $\bar x^i = y^i \in\{(1,0)^\top,~(-1,0)^\top\}$. Observe that for this loss function the set of minimizers is given by a singleton $\mbf{X}(\mbf{y}) = \{\bar{\mbf{x}} = \mbf{y}\}$. We set $ \gamma  = 10^2 D$.
The activation function is chosen as $\sigma(s) = \tanh(s)$.
Solving the TST example for $N=10$ and 300 noisy data points from Figure~\ref{fig:xygen}, we obtain the solutions depicted in Figure~\ref{fig:xy} and Figure~\ref{fig:u}. 
As one can see, the solution of OCP \eqref{eq:OCP} attains near zero loss. 
Moreover, the generalization properties of the trained DNN are illustrated in Figures~\ref{fig:xygen_k5_ideal},~\ref{fig:xygen_k5}, and~\ref{fig:xygen}.  
For  sample data $\mbb{D}$ has been perturbed by noise. 
Figures~\ref{fig:xygen_k5_ideal},~\ref{fig:xygen_k5}, and~\ref{fig:xygen} depict the decision-boundary as well the classification over a set within the boundaries $-2$ to $2$ for the case where the number of data points $D=300$ is kept constant.
For comparison the unperturbed data and its classification are given in Figures~\ref{fig:xygen_k5_ideal}. 
In all examples the decision-boundary remains qualitatively similar, networks with more layers tend to give a better generalization, see the areas, which are unsupported with learning data in Figure~\ref{fig:xygen} or compare to ideal data Figure~\ref{fig:xygen_k5_ideal}.

\subsection{Depth Bounds from A-posteriori Estimates}
As the solution of OCP \eqref{eq:OCP} delivers (near) zero loss on the training data, we now compare three different approaches to a-posteriori depth bounds: (i) we estimate the exponential constants $\beta, \rho>0$ from Assumption~\ref{ass:ExpReach} using the computed optimal  trajectories and use \eqref{eq:N}; (ii)  we use the bound from Proposition \ref{prop:infnorm} in  \eqref{eq:NfromJ} with the two norm; and (iii) we consider  \eqref{eq:NfromJ} for trajectorties computed via the two norm but the bound is evaluated in the $\infty$ norm. In other words,  (iii) leverages the result of Proposition~\ref{prop:infnorm} to estimate the depth based on the trajectories computed for $\pi = (2, 2, 2, 2, 10^{-2}, 10^{-3})$ in the metric implied by $\pi = (1, 2, \infty, 2, 1, 10^{-3})$. The $\infty$-norm penalization of the state deviation is motivated by the $\infty$ norm not growing with the cardinality of the data set $\mbb{D}$.
The depth bound form (i) is denoted as $\hat N_2(\beta, \rho)$, the one from (ii) as $\hat N_2$, and the one from (iii) as  $\hat N_\infty$.

To the end of estimating $\beta$ and $\rho$, we solve the NLP:
\begin{align} \label{eq:estN}
\min_{\beta, \rho \in \mbb{R}^+_0} &~\dfrac{\beta}{1-\rho} \nonumber
\text{ subject to }
\ell(\mbf{x}^\star_k, \mbf{u}^\star_k) \leq \beta\rho^k, \quad k \in \mbb{N}_{[0, N-1]}.
\end{align}

Observe that in the TST example classification is achieved if a data point propagated through the net reaches the interior of $\mcl{N}_{\delta=1}(\bar{\mbf{x}})$, cf. Proposition~\ref{prop:0lossClass}. Hence we set $\varepsilon=1$ in evaluating $\hat{N}$ in \eqref{eq:N} and \eqref{eq:NfromJ}. Also note that due to the choice $q = 10^{-2}$ in \eqref{eq:LQell} we have $\alpha_\ell(\varepsilon) = 10^{-2}$ for $\hat N_2(\beta, \rho)$ and $\hat N_2$. Table~\ref{tab:N} summarizes the constants $\beta$ and $\rho$ computed via the NLP above as well as the estimated depth bound $\hat{N}_2$ for variations of the network depth (used to solve \eqref{eq:OCP}) and the number of considered data points without noise.   

\begin{table}
\caption{Estimated net depths via \eqref{eq:N} for noise-free data.\label{tab:N}}
\begin{center}
\begin{tabular}{c  | c || c | c | c | c | c}
\hline
$N$ & $D$ & $\beta$  & $ \rho$ & $\hat{N}_2(\beta,\rho)$&$\hat{N}_2$&  $\hat{N}_\infty\phantom{\Big|}$\\
\hline
$5 $&$ 20	$&$ 0.75 $&$ 0.61 $&$ 2.42\cdot 10^2 $&$ 1.01\cdot10^2 $&$ \phantom{0}6.31$ \\
$5 $&$ 50	$&$ 1.85 $&$ 0.83 $&$ 1.11\cdot 10^3 $&$3.21\cdot10^2 $&$  \phantom{0}6.74$\\
$5 $&$ 100	$&$ 3.57 $&$ 0.83 $&$ 2.14\cdot 10^3$&$6.11\cdot10^2 $&$ \phantom{0}6.98$\\
$5 $&$ 250	$&$ 8.43 $&$ 0.67 $&$ 2.56\cdot 10^3 $&$1.23\cdot10^3 $&$ \phantom{0}6.56$\\
$5 $&$ 500	$&$ 14.4 $&$ 0.79 $&$  6.83\cdot 10^3  $&$3.36\cdot10^3 $&$ \phantom{0}7.53$\\
\hline
$10 $&$ 20	 $&$ 0.75 $&$ 0.60 $&$ 2.42\cdot 10^2 $&$1.09\cdot10^2 $&$ \phantom{0}7.00 $  \\
$10 $&$ 50	 $&$ 2.37 $&$ 0.68 $&$ 7.46\cdot 10^2 $&$3.69\cdot10^2 $&$ 11.06 $\\
$10 $&$ 100 $&$3.08 $&$ 0.73 $&$ 1.14\cdot 10^3 $&$5.34\cdot10^2 $&$ \phantom{0}8.39$\\
$10 $&$ 250 $&$ 7.34 $&$ 0.77 $&$ 3.20\cdot 10^3 $&$1.44\cdot10^3 $&$ \phantom{0}9.83$\\
$10 $&$ 500 $&$ 19.5 $&$ 0.58 $&$ 4.67\cdot 10^3 $&$2.05\cdot10^3 $&$ \phantom{0}7.01$\\
\hline
\end{tabular}
\end{center}
\end{table}
\begin{table}
\caption{Estimated net depths  via \eqref{eq:N} for noisy data.\label{tab:Nnoise}}
\begin{center}
\begin{tabular}{c  | c || c | c | c | c | c }
\hline
$N$ & $D$ & $\beta$  & $ \rho$ & $\hat{N}_2(\beta,\rho)$&$\hat{N}_2$&  $\hat{N}_\infty\phantom{\Big|}$\\
\hline
$5 $&$ 20	$&$ 0.87$&$ 0.83 $&$ 5.24\cdot 10^2 $&$1.45\cdot10^2 $&$ \phantom{0}7.42$ \\
$5 $&$ 50	$&$ 1.48 $&$ 0.83 $&$8.88\cdot 10^2 $&$2.88\cdot10^2 $&$  \phantom{0}7.20$\\
$5 $&$ 100	$&$ 3.67 $&$ 0.83 $&$ 2.20\cdot 10^3$&$6.27\cdot10^2 $&$ \phantom{0}7.15$\\
$5 $&$ 250	$&$ 8.83$&$ 0.83 $&$ 5.30\cdot 10^3 $&$1.58\cdot10^3 $&$ \phantom{0}8.57$\\
$5 $&$ 500	$&$ 15.9 $&$ 0.83 $&$  9.53\cdot 10^3  $&$3.56\cdot10^3 $&$ \phantom{0}8.40$\\
\hline
$10 $&$ 20	$&$ 1.01$&$ 0.75 $&$ 4.02\cdot 10^2 $&$1.82\cdot10^2 $&$ 11.50$ \\
$10 $&$ 50	$&$ 1.65 $&$ 0.68 $&$ 5.07\cdot 10^2 $&$2.35\cdot10^2 $&$  \phantom{0}8.02$\\
$10 $&$ 100	$&$ 2.86 $&$ 0.79 $&$ 1.34\cdot 10^3$&$5.89\cdot10^2 $&$ 11.38$\\
$10 $&$ 250	$&$ 7.91 $&$ 0.72 $&$ 2.80\cdot 10^3 $&$1.11\cdot10^3 $&$ \phantom{0}8.50$\\
$10 $&$ 500	$&$ 13.8 $&$ 0.83 $&$  7.94\cdot 10^3  $&$3.30\cdot10^3 $&$ 18.75$\\
\hline
\end{tabular}
\end{center}
\end{table}

As one can see, the obtained a-posteriori depth estimates $\hat N_2(\beta, \rho)$ range from $242$ to $6830$ if the two norm is used. Observe that the bound $\hat N_2(\beta, \rho)$ grows with  increasing cardinality of the data set $\#\mbb{D} =D$. The bound $\hat N_2$ which relies on \eqref{eq:NfromJ} is only marginally better ranging from $101$ to $3360$. In contrast the bound $\hat N_\infty$, which is based on the $\infty$-norm state penalization in $\ell$, delivers much smaller estimates ranging from $6.31$ to $11.06$. 
The trends on  $\hat N_2(\beta, \rho)$ and $\hat N_2$ remain unchanged in Table~\ref{tab:Nnoise} wherein noisy data sets $\mbb{D}$ are considered. The range of $\hat N_2(\beta, \rho)$ spans $524$ to $7940$ and $\hat N_2$ is between $145$ and $3560$. In contrast, $\hat N_\infty$ spans from $7.20$ to $18.8$. Again the $\infty$-norm state penalization in $\ell$ delivers much smaller estimates. Observe that the trend of $\hat N$ growing with $D$ is much less pronounced for $\hat N_\infty$. The average of values is centered around $7-8$. The outliers $>11$ and $18.75$ are likely due to the training problem \eqref{eq:OCP} being solved to local optimality. 
Finally, the trend that for increasing number of samples $D$ the estimated bounds $\hat{N}_2$, but also $\hat N_\infty$, increase, indicates that the reachability properties of the ensemble dynamics \eqref{eq:stack_sys} are  affected by the dimensionality of the stacked system state $\dim(\mbf{x}) = D$.  Yet, in view of the generalization plot in Figure \ref{fig:xygen}, we conclude that the reachability of $\bar{\mbf{x}}\in\mbf{X}^\star(\mbf{y})$ is sufficient but not necessary for classification. 

\section{Discussion and Open Problems}\label{sec:conc}\vspace*{-2mm}
This note has taken steps towards deriving depth bounds for deep neural networks via turnpike and dissipativity theory. 
The example of the previous section has  assessed the quality of the depth bound derived in Theorem~\ref{thm:epsloss}. 
The a-posteriori results of Tables~\ref{tab:N} and \ref{tab:Nnoise} indicate that the structured design of DNN is accessible to analysis techniques derived in context of turnpike properties of optimal control problems. 
However,  several issues require further and future research.

Arguably the most pressing question is how to extend the approach from  a-posteriori computation of $\beta$ and $\rho$ based on a trained DNN to a-priori estimation based on a given loss function $\ell_\mathrm{f}$, the activation function $\sigma$, and available data? Moreover, the results from Tables~\ref{tab:N} and \ref{tab:Nnoise} indicate that estimates could be obtained by considering only parts of the data $\mbb{D}$, similarly to stochastic gradient techniques, which also rely only on partial data sets. Moreover, we have shown how one can use numerically favorable squared two norms in the regularization stage cost, while building the depth estimates via the $\infty$ norm to avoid scaling with the number of data points. 

One key feature driving the success of DNNs in ML applications are their generalization properties, i.e., DNNs provide reasonable classification/regression capabilities for data points not contained in the training data $\mbb{D}$. Yet, it is not fully clear how to design the loss function $\ell_\mathrm{f}$ and the stage cost regularization $\ell$ to foster generalization. As mentioned at the end of Section~\ref{sec:example}, reachability of $\bar{\mbf{x}}$ is sufficient for classification but not necessary. Hence there is evident need for further analysis on the choice of $\bar{\mbf{x}}$ or on the choice of more general formulations of $\ell$. Intuitively, the OCP formulation of the training problem also suggests the analysis of robustness properties of solutions to \eqref{eq:sys}. 

Finally, the presented numerical results are a first indication of the potential of  systems theory approaches towards the analysis of DNN training via OCPs. Evidently, further numerical examples with other loss functions and considering established stochastic optimization methods (e.g. stochastic gradient methods which only consider a subset of the training data in each optimization step) call for future work.

\printbibliography

\end{document}